\documentclass{article}

\usepackage{microtype}
\usepackage{graphicx}
\usepackage{booktabs} 
\usepackage{tabularx}
\usepackage{siunitx}

\usepackage{amsfonts,amsmath,amsthm}

\usepackage{algorithm,algorithmic}
\usepackage{subcaption}
\usepackage{todonotes}

\usepackage{hyperref}

\makeatletter
\newtheorem*{rep@theorem}{\rep@title}
\newcommand{\newreptheorem}[2]{%
\newenvironment{rep#1}[1]{%
 \def\rep@title{#2 \ref{##1}}%
 \begin{rep@theorem}}%
 {\end{rep@theorem}}}
\makeatother

\newtheorem{theorem}{Theorem}
\newreptheorem{theorem}{Theorem}

\newcommand{\argmax}{\mathrm{argmax}}
\newcommand{\argmin}{\mathrm{argmin}}

\newcommand{\alglinelabel}{%
  \addtocounter{ALC@line}{-1}
  \refstepcounter{ALC@line}
  \label
}

\usepackage[preprint]{icml2021}

\icmltitlerunning{Q-Value Weighted Regression}

\begin{document}

\twocolumn[
\icmltitle{Q-Value Weighted Regression: Reinforcement Learning with Limited Data}



\icmlsetsymbol{equal}{*}

\begin{icmlauthorlist}
\icmlauthor{Piotr Kozakowski}{equal,goo,mim}
\icmlauthor{Łukasz Kaiser}{equal,brain}
\icmlauthor{Henryk Michalewski}{goo,mim}
\icmlauthor{Afroz Mohiuddin}{brain}
\icmlauthor{Katarzyna Kańska}{goo}
\end{icmlauthorlist}

\icmlaffiliation{mim}{Faculty of Mathematics, Informatics and Mechanics, University of Warsaw, Poland}
\icmlaffiliation{goo}{Google}
\icmlaffiliation{brain}{Google Brain}

\icmlcorrespondingauthor{Piotr Kozakowski}{p.kozakowski@mimuw.edu.pl}

\icmlkeywords{Machine Learning, ICML}

\vskip 0.3in
]



\printAffiliationsAndNotice{\icmlEqualContribution} 

\begin{abstract}
Sample efficiency and performance in the offline setting have emerged as significant
challenges of deep reinforcement learning. We introduce Q-Value Weighted Regression (QWR),
a simple RL algorithm that excels in these aspects.
QWR is an extension of Advantage Weighted Regression (AWR), an off-policy actor-critic algorithm
that performs very well on continuous control tasks, also in the offline setting, but has low
sample efficiency and struggles with high-dimensional observation spaces. We perform
an analysis of AWR that explains its shortcomings and use these insights to motivate QWR.
%
We show experimentally that QWR matches the state-of-the-art algorithms both on tasks with
continuous and discrete actions.
In particular, QWR yields results on par with SAC on the MuJoCo suite and -- with
the same set of hyperparameters -- yields results on par with a highly tuned Rainbow
implementation on a set of Atari games. We also verify that QWR performs well in the
offline RL setting.


\end{abstract}

\section{Introduction}


Deep reinforcement learning has been applied to a large number of challenging tasks,
from games \citep{silver2017mastering,OpenAI_dota,vinyals2017starcraft} to robotic control \citep{sadeghi2016cad2rl,openai2018dexterous,rusu2016sim2real}. Since RL makes minimal assumptions
on the underlying task, it holds the promise of automating a wide range of applications.
However, its widespread adoption has been hampered by a number of challenges.
Reinforcement learning algorithms can be substantially more complex to implement
and tune than standard supervised learning methods and can have a fair number of hyper-parameters
and be brittle with respect to their choices, and may require a large number of interactions with
the environment.

These issues are well-known and there has been significant progress in addressing them. 
The policy gradient algorithm
REINFORCE \citep{reinforce1} is simple to understand and implement, but is both brittle and requires on-policy
data. Proximal Policy Optimization (PPO, \citet{schulman2017proximal}) is a more stable on-policy algorithm that has
seen a number of successful applications despite requiring a large number of interactions
with the environment. Soft Actor-Critic (SAC, \citet{haarnoja2018soft}) is a much more sample-efficient off-policy algorithm, but it is defined only for continuous action spaces and 
does not work well in the offline setting, known as batch reinforcement learning, where all samples
are provided from earlier interactions with the environment, and the agent cannot collect more samples.
Advantage Weighted Regression
(AWR, \citet{peng2019advantageweighted})
is a recent off-policy actor-critic algorithm that works well in the offline setting
and is built using only simple and convergent maximum likelihood loss functions, making it
easier to tune and debug. It is competitive with SAC given enough time to train,
but is less sample-efficient and has not been demonstrated to succeed in settings with discrete actions.

We replace the value function critic of AWR with a Q-value function.
Next, we add action sampling to the actor training loop. Finally, we introduce a custom backup to the Q-value training.
The resulting algorithm, which we call
Q-Value Weighted Regression (QWR) inherits the advantages of AWR but is more sample-efficient
and works well with discrete actions and in visual domains, e.g., on Atari games.

To better understand QWR we perform a number of ablations, checking different number of samples in actor training, different advantage estimators, and aggregation functions. These choices affect
the performance of QWR only to a limited extent and it remains stable with each of the choices across the tasks we experiment with.

We run experiments with QWR on the MuJoCo environments and on a subset of the Arcade Learning
Environment. Since sample efficiency is our main concern, we focus on the difficult case when
the number of interactions with the environment is limited -- in most our experiments we
limit it to 100K interactions. The experiments demonstrate that QWR is indeed more
sample-efficient than AWR. On MuJoCo, it performs on par with Soft Actor-Critic (SAC),
the current state-of-the-art algorithm for continuous domains.
On Atari, QWR performs on par with OTRainbow, a variant of Rainbow highly tuned for
sample efficiency. Notably, we use the same set of hyperparameters (except for the network
architecture) for both MuJoCo and Atari experiments.
We verify that QWR performs well also in the regime where more data is
available: with 1M interactions, QWR still out-perform SAC on MuJoCo
on all environments we tested except for HalfCheetah.


\section{Q-Value Weighted Regression}
\label{sec:qwr}

\subsection{Advantage Weighted Regression}

\citet{peng2019advantageweighted} recently proposed Advantage Weighted Regression (AWR), an off-policy, actor-critic algorithm notable for its simplicity and stability, achieving competitive results across a range of continuous control tasks. It can be expressed as interleaving data collection and two regression tasks performed on the replay buffer, as shown in \autoref{algo:awr}.


\begin{algorithm}[H]
\begin{algorithmic}[1]
\STATE $\theta \gets$ random actor parameters
\STATE $\phi \gets$ random critic parameters
\STATE $\mathcal{D} \gets \emptyset$
\FOR{$k$ \textbf{in} $0..n\_iterations-1$}
\STATE add trajectories $\{\tau_i\}$ sampled by $\pi_\theta$ to $\mathcal{D}$

\FOR{$i$ \textbf{in} $0..n\_critic\_steps-1$}
\STATE sample $(\mathbf{s, a}) \sim \mathcal{D}$
\STATE $\phi \gets \phi - \alpha_V \nabla_\phi ||\mathcal{R}_\mathcal{D}^{\mathbf{s, a}} - V_\phi(\mathbf{s}) ||^2$
\ENDFOR
\FOR{$i$ \textbf{in} $0..n\_actor\_steps-1$}
\STATE sample $(\mathbf{s, a}) \sim \mathcal{D}$
\STATE $\xi \gets \exp(\frac{1}{\beta} (\mathcal{R}_\mathcal{D}^{\mathbf{s, a}} - V_\phi(\mathbf{s}))$
\STATE $\theta \gets \theta + \alpha_\pi \nabla_\theta \log \pi_\theta(\mathbf{a | s})~\xi$
\ENDFOR
\ENDFOR
\end{algorithmic}
\caption{Advantage Weighted Regression.}
\label{algo:awr}
\end{algorithm}

AWR optimizes \textit{expected improvement} of an actor policy $\pi(\mathbf{a} | \mathbf{s})$ over a sampling policy $\mu(\mathbf{a} | \mathbf{s})$ by regression towards the well-performing actions in the collected experience. Improvement is achieved by weighting the actor loss by exponentiated advantage $A_\mu(\mathbf{s, a})$ of an action, skewing the regression towards the better-performing actions. 
The advantage is calculated based on the expected return $\mathcal{R}_\mu^{\mathbf{s, a}}$ achieved by performing action $\mathbf{a}$ in state $\mathbf{s}$ and then following the sampling policy $\mu$. To calculate the advantage, one first estimates the value,  $V_\mu(s)$, using a learned critic and then computes $A_\mu(\mathbf{s, a}) = \mathcal{R}_\mu^{\mathbf{s, a}} - V_\mu(\mathbf{s})$.
This results in the following formula for the actor:
\begin{align} \label{eq:awr_actor_loss}
\begin{split}
\argmax_\pi \mathbb{E}_{\mathbf{s} \sim d_\mu} \mathbb{E}_{\mathbf{a} \sim \mu(\cdot | \mathbf{s})} \log \pi(\mathbf{a} | \mathbf{s})~\xi^{\mathbf{s, a}}_\mu, & \\
\textrm{where}~\xi^{\mathbf{s, a}}_\mu = \exp \left( \frac{1}{\beta}(\mathcal{R}_\mu^{\mathbf{s, a}} - V_\mu(\mathbf{s})) \right). &
\end{split}
\end{align}
%
%
In this formula $d_\mu(\mathbf{s}) = \sum_{t=1}^{\infty}\gamma^{t - 1} p(\mathbf{s}_t = \mathbf{s}|\mu)$ denotes the unnormalized, discounted state visitation distribution of the policy $\mu$, and $\beta$ is a temperature hyperparameter.

The critic is trained to estimate the future returns of the sampling policy $\mu$:
\begin{equation}
\argmin_V \mathbb{E}_{\mathbf{s} \sim d_\mu(\mathbf{s})} \mathbb{E}_{\mathbf{a} \sim \mu(\cdot | \mathbf{s})} \left[ || \mathcal{R}_\mu^{\mathbf{s, a}} - V(\mathbf{s}) ||^2 \right].
\label{eq:awr_critic_loss}
\end{equation}
To achieve off-policy learning, the actor and the critic are trained on data collected from a mixture of policies from different training iterations, stored in the replay buffer $\mathcal{D}$. 

\subsection{Analysis of AWR with Limited Data}

While AWR achieves very good results after longer training, it is not very sample efficient,
as noted in the future work section of \citet{peng2019advantageweighted}. To understand this
problem, we analyze a single loop of actor training in AWR under a special assumption.

The assumption we introduce, called \emph{state-determines-action}, concerns the content of the
replay buffer $\mathcal{D}$ of an off-policy RL algorithm. The replay buffer contains the state-action pairs
that the algorithm has visited so far during its interactions with the environment. We say that
a replay buffer $\mathcal{D}$ satisfies the \emph{state-determines-action} assumption when for each
state $s$ in the buffer, there is a unique action that was taken from it, formally:
\[ \text{for all } (s, a), (s', a') \in \mathcal{D}: s = s' \implies a = a'. \]

This assumption may seem limiting and indeed -- it is not true in many of the artificial experiments with RL algorithms, with discrete state and action spaces. In such settings, even a random policy starting from the same state could violate the assumption the second time it collects a trajectory. But note that state-determines-action is almost always satisfied in continuous control, where even a slightly random policy is unlikely to ever perform the exact same action twice and transition to \emph{exactly} the same state.

Note that our assumption applies well to real-world experiments with high-dimensional state spaces, as any amount of noise added to a high-dimensional space will make repeating the exact same state highly improbable. For example, consider a robot observing 32x32 pixel images. To repeat an observation, each of the 1024 pixels would have to have exactly the same value, which is close to impossible, even with a small amount of pixel noise coming from a camera. This assumption also holds in cases with limited data, even in discrete state and action spaces. When the number of collected trajectories is not enough to span the state space, it is unlikely a state will be repeated in the replay buffer. This makes our assumption particularly relevant to the study of sample efficiency.

We emphasize that the state-determines-action assumption, by design, considers \textit{exact equality} of states. Two very similar, but not equal states that lead to different actions do not violate our assumption. This makes it irrelevant to reinforcement learning with linear functions
as linear functions cannot separate similar states.
However, it is relevant in deep RL because deep neural networks can indeed distinguish even very similar inputs  \citep{adv-examples,margins,understanding-generalization}. 


How does AWR perform under the state-determines-action assumption? In Theorems~\ref{thm:awr_disc} and \ref{thm:awr_cont} (see Appendix~\ref{sec:limdata} for more details), we show that for popular choices of discrete and Gaussian distributions the AWR update rule
under this assumption will converge to a policy that assigns probability 1 to the actions
already present in the replay buffer, thus cloning the previous behaviors. This is not the desired
behavior, as an agent should consider various actions from each state, to ensure exploration.
\begin{reptheorem}{thm:awr_disc}
Let $\mathcal{A}$ be a discrete action space. Let a replay buffer $\mathcal{D} \subseteq \mathcal{S} \times \mathcal{A}$ satisfy the state-determines-action assumption. Let $\pi_\mathcal{D}$ be the probability function of a distribution that clones the behavior from $\mathcal{D}$, i.e., that assigns to each state $s$ from $\mathcal{D}$ the action $a$ such that
$(s, a) \in \mathcal{D}$ with probability $1$.  Then, under the AWR update,
$\pi^{i + 1}_\mathcal{D} \gets \pi_\mathcal{D}$.
\end{reptheorem}


The state-determines-action assumption is the main motivating point behind QWR, whose theoretical properties are proven in Theorem~\ref{thm:qwr} in Appendix~\ref{sec:limdata-qwr}. We now illustrate the importance of this assumption by creating a simple environment in which it holds with high probability. We verify experimentally that AWR fails on this simple environment, while QWR is capable of solving it.

\begin{figure}[t]
  \centering
  \includegraphics[width=1\linewidth]{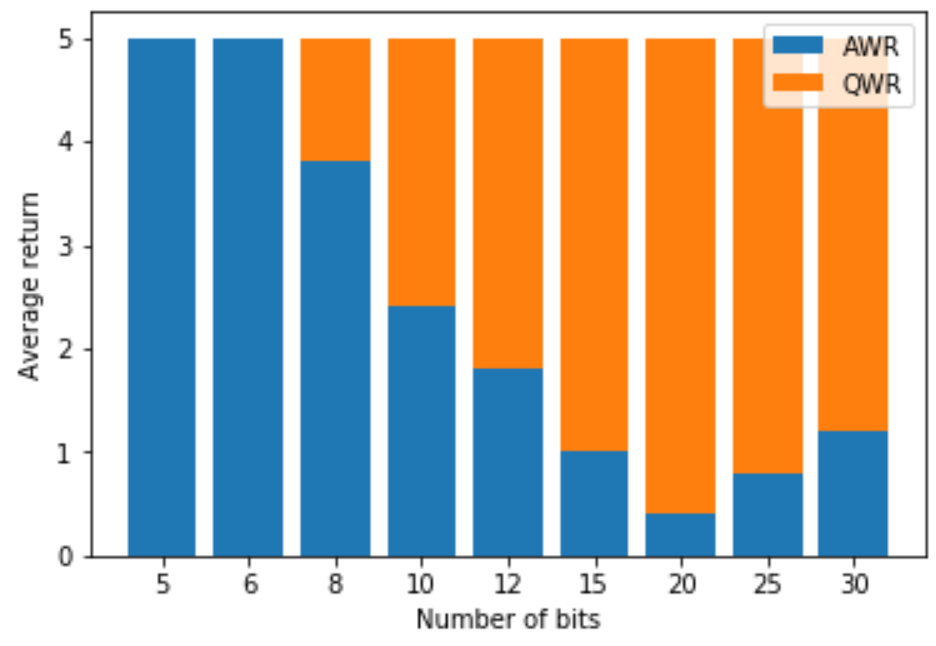}
  \caption{AWR and QWR on the BitFlip environment. The maximum possible return is 5.}
  \label{fig:bitflip}
\end{figure}

The environment, which we call \emph{BitFlip}, is parameterized by an integer $N$. The state
of the environment consists of $N$ bits and a step counter. The action space consists of $N$ actions.
When an action $i$ is chosen, the $i$-th bit is flipped and the step counter is incremented. A game of BitFlip starts in a random state with the step counter set to 0,
and proceeds for 5 steps. The initial state is randomized in such a way to always leave at least 5 bits set to 0. At each step, the reward is $+1$ if a bit was flipped
from $0$ to $1$ and the reward is $-1$ in the opposite case.

Since BitFlip starts in one random state out of $2^N$, at large enough $N$ it is highly unlikely that the starting
state will ever be repeated in the replay buffer. As the initial policy is random and BitFlip
maintains a step counter to prevent returning to a state, the same holds for subsequent states.

BitFlip is a simple game with a very simple strategy, but the initial replay buffer will
satisfy the state-determines-action assumption with high probability. As we will see, this is enough to break AWR. We ran both AWR and QWR on BitFlip for different values of $N$, for 10 iterations per experiment. In each iteration we collected 1000 interactions with the environment and trained both the actor and the critic for 300 steps. All shared hyperparameters of AWR and QWR were set to the same values, and the backup operator in QWR was set to mean. We report the mean out of 10 episodes played by the trained agent. The results are shown in \autoref{fig:bitflip}.

As we can see, the performance of AWR starts deteriorating at a relatively small value of $N~=~8$, which corresponds to a state space with $5 \cdot 2^{8} = 1280$ states, while QWR maintains high performance even at $N=30$, so around $5 \cdot 10^9$ states. Notice how the returns of AWR drop with $N$ -- at higher values: $20-30$, the agent struggles to flip even a single zero bit. This problem with AWR and large state spaces motivates us to introduce QWR next.


\subsection{Q-Value Weighted Regression}
\label{qwr:intro}

\begin{algorithm}[h]
\begin{algorithmic}[1]
\STATE $\theta \gets$ random actor parameters
\STATE $\phi \gets$ random critic parameters
\STATE $\mathcal{D} \gets \emptyset$
\FOR{$k$ \textbf{in} $0..n\_iterations-1$}
\STATE add trajectories $\{\tau_i\}$ sampled by $\pi_\theta$ to $\mathcal{D}$

\STATE $\phi_t \gets \phi$
\FOR{$i$ \textbf{in} $0..n\_critic\_steps-1$}
\IF{$i\mod update\_frequency = 0$}
\STATE $\phi_t \gets \phi$
\ENDIF
\STATE sample $(\mathbf{s}, \mathbf{\mu}, \mathbf{a, r, s'}) \sim \mathcal{D}$
\STATE sample $\mathbf{a'}_0, ..., \mathbf{a'}_{n-1} \sim \mathbf{\mu}(\cdot | \mathbf{s'})$ \alglinelabel{ln:sampling-critic}
\STATE $Q^\star \gets \mathbf{r} + \gamma F(\{Q_{\phi_t}(\mathbf{s', a'}_j)~|~j \in \{0, ..., n - 1\}\})$\alglinelabel{ln:backup}
\STATE $\phi \gets \phi - \alpha_V \nabla_\phi || Q^\star - Q_\phi(\mathbf{s, a}) ||^2$\alglinelabel{ln:q-learning}
\ENDFOR
\FOR{$i$ \textbf{in} $0..n\_actor\_steps-1$}\alglinelabel{ln:actor-begin}
\STATE sample $(\mathbf{s}, \mu, ...) \sim \mathcal{D}$
\STATE sample $\mathbf{a}_0, ..., \mathbf{a}_{n-1} \sim \mu(\cdot | \mathbf{s})$\alglinelabel{ln:sampling-actor}
\STATE $\hat{V} \gets \frac{1}{n} \sum_{j = 0}^{n-1} Q_\phi(\mathbf{s, a}_j)$
\STATE $\xi \gets \exp(\frac{1}{\beta} (Q_\phi(\mathbf{s, a}_j) - \hat{V}))$
\STATE $\theta \gets \theta + \alpha_\pi \nabla_\theta \frac{1}{n} \sum_{j = 0}^{n-1} \log \pi_\theta(\mathbf{a}_j | \mathbf{s})~\xi$\alglinelabel{ln:actor}
\ENDFOR\alglinelabel{ln:actor-end}
\ENDFOR
\end{algorithmic}

\caption{Q-Value Weighted Regression.}
\label{algo:qwr_gradient}
\end{algorithm}

To remedy the issue indicated by Theorems~\ref{thm:awr_disc} and \ref{thm:awr_cont}, we introduce a mechanism to consider multiple different actions that can be taken from a single state. We calculate the advantage of the sampling policy $\mu$ based on a learned Q-function: $A_\mu(\mathbf{s, a}) = Q_\mu(\mathbf{s, a}) - \hat{V}_\mu(\mathbf{s})$, where $\hat{V}_\mu(\mathbf{s})$ is the expected return of the policy $\mu$, expressed using $Q_\mu$ by expectation over actions: $\hat{V}_\mu(\mathbf{s})~=~\mathbb{E}_{a \sim \mu(\cdot | \mathbf{s})} Q_\mu(\mathbf{s, a})$. We substitute our advantage estimator into the AWR actor formula (\autoref{eq:awr_actor_loss}) to obtain the QWR actor:
\begin{align}
\begin{split}
\argmax_\pi \mathbb{E}_{s \sim d_\mu(\mathbf{s})} \mathbb{E}_{\mathbf{a} \sim \mu(\cdot | \mathbf{s})} \log \pi(\mathbf{a} | \mathbf{s})~\xi_\mu^\mathbf{s,a}, & \\
\textrm{where}~\xi_\mu^\mathbf{s,a} = \exp\left(\frac{1}{\beta} (Q_\mu(\mathbf{s, a}) - \hat{V}_\mu(\mathbf{s}))\right). &
\label{eq:qwr_actor_loss}
\end{split}
\end{align}


Similar to AWR, we implement the expectation over states in \autoref{eq:qwr_actor_loss} by sampling from the replay buffer. However, to estimate the expectation over actions, we average over multiple actions sampled from $\mu$ during training. Because the replay buffer contains data from multiple different sampling policies, we store the parameters of the sampling policy $\mu(\mathbf{a} | \mathbf{s})$  conditioned on the current state in the replay buffer and restore it in each training step to compute the loss. This allows us to consider multiple different possible actions for a single state when training the actor, not only the one performed in the collected experience.

The use of a Q-network as a critic provides us with an additional benefit. Instead of regressing it towards the returns of our sampling policy $\mu$, we can train it to estimate the returns of an improved policy $\mu^\star$, in a manner similar to Q-learning. This allows us to optimize expected improvement over $\mu^\star$, providing a better baseline - as long as  $\mathbb{E}_{\mathbf{a} \sim \mu^\star(\cdot | \mathbf{s})} Q_\mu(\mathbf{s}, \mathbf{a}) \geq \mathbb{E}_{\mathbf{a} \sim \mu(\cdot | \mathbf{s})} Q_\mu(\mathbf{s, a})$, the \textit{policy improvement theorem} for stochastic policies \citep[Section 4.2]{sutton_barto} implies that the policy $\mu^\star$ achieves higher returns than the sampling policy $\mu$:
%
\begin{equation}
\mathbb{E}_{\mathbf{a} \sim \mu^\star(\cdot | \mathbf{s})} Q_\mu(\mathbf{s, a}) \geq V_\mu(\mathbf{s}) \Rightarrow V_{\mu^\star}(\mathbf{s}) \geq V_\mu(\mathbf{s})
\end{equation}

$\mu^\star$ need not be parametric - in fact, it is not materialized in any way over the course of the algorithm. The only requirement is that we can estimate the Q backup $\mathbb{E}_{\mathbf{a} \sim \mu^\star(\cdot | \mathbf{s})} Q(\mathbf{s, a})$. This allows great flexibility in choosing the form of $\mu^\star$. Since we want our method to work also in continuous action spaces, we cannot compute the backup exactly. Instead, we estimate it based on several samples from the sampling policy $\mu$. Our backup has the form $\mathbb{E}_{\mathbf{a}_1, ..., \mathbf{a}_k \sim \mu(\cdot | \mathbf{s})} F(\{ Q(\mathbf{s, a}_1), ..., Q(\mathbf{s, a}_k) \})$. In this work, we extend the term \textit{Q-learning} to mean training a Q-value using such a generalized backup. To make training of the Q-network more efficient, we use multi-step targets, described in detail in Appendix~\ref{sec:multistep}. The critic optimization objective using single-step targets is:
\begin{align} \label{eq:qwr_critic_loss}
\begin{split}
\argmin_Q~&\mathbb{E}_{\mathbf{s} \sim d_\mu(\mathbf{s})} \mathbb{E}_{\mathbf{a} \sim \mu(\mathbf{a} | \mathbf{s})} \mathbb{E}_{\mathbf{s'} \sim \mathcal{T}(\cdot | \mathbf{s, a})} \\ 
&\mathbb{E}_{\mathbf{a'}_1, ..., \mathbf{a'}_k \sim \mu(\cdot | \mathbf{s'})} || Q^\star - Q(\mathbf{s, a}) ||^2,
\end{split}
\end{align}
where
\begin{align*}
Q^\star = r(\mathbf{s, a}) + \gamma F(\{Q_\mu(\mathbf{s', a'}_1), ..., Q_\mu(\mathbf{s', a'}_k)\}) \\
\end{align*}
and $\mathcal{T}(\mathbf{s'} | \mathbf{s, a})$ is the environment's transition distribution.

In this work, we investigate three choices of $F$: average, yielding $\mu^\star = \mu$; max, where $\mu^\star$ approximates the greedy policy; and log-sum-exp, $F(X) = \tau \log \left[ \frac{1}{|X|} \sum_{x \in X} \exp(x / \tau) \right]$, interpolating between average and max with the temperature parameter $\tau$. This leads to three versions of the QWR algorithm:  QWR-AVG, QWR-MAX, and QWR-LSE.  The last operator, log-sum-exp, is similar to the backup operator used in maximum-entropy reinforcement learning (see e.g. \citet{haarnoja2018soft}) and can be thought of as a soft-greedy backup, rewarding both high returns and uncertainty of the policy. It is our default choice and
the final algorithm is shown in \autoref{algo:qwr_gradient}.


\section{Related work}



\paragraph{Reinforcement learning algorithms.}
Recent years have seen great advances in the field of reinforcement learning due to the use of deep neural networks as function approximators. \citet{mnih2013playing} introduced DQN, an off-policy algorithm learning a parametrized Q-value function through updates based on the Bellman equation. 
The DQN algorithm only computes the Q-value function, it does not learn an explicit policy. 
In contrast, policy-based methods such as REINFORCE \citep{reinforce1} learn a parameterized policy, typically by following the policy gradient \citep{reinforce2} estimated through Monte Carlo approximation of future returns. Such methods suffer from high variance, causing low sample efficiency. Actor-critic algorithms, such as A2C and A3C \citep{sutton1999a2c, mnih2016asynchronous}, decrease the variance of the estimate by jointly learning policy and value functions, and using the latter as an action-independent baseline for calculation of the policy gradient. The PPO algorithm \citep{schulman2017proximal} optimizes a clipped surrogate objective in order to allow multiple updates using the same sampled data.

\paragraph{Continuous control.}
\citet{lillicrap2015continuous} adapted Q-learning to continuous action spaces. In addition to a Q-value function, they learn a deterministic policy function optimized by backpropagating the gradient through the Q-value function. \citet{haarnoja2018soft} introduce Soft Actor-Critic (SAC): a method learning in a similar way, but with a stochastic policy optimizing the Maximum Entropy RL \citep{levine2018reinforcement} objective. Similarly to our method, SAC also samples from the policy during training.

\paragraph{Advantage-weighted regression.} The QWR algorithm is a successor of AWR proposed by \citet{peng2019advantageweighted}, which in turn is based on Reward-Weighted Regression (RWR, \citet{rwr}) and AC-REPS proposed by \citet{wirth2016acreps}. Mathematical and algorithmical foundations of advantage-weighted regression were developed by \citet{fqi}. The algorithms share the same good theoretical properties: RWR, AC-REPS, AWR, and QWR losses can be mathematically reformulated in terms of KL-divergence with respect to the optimal policy (see formulas (7)-(10) in \citet{peng2019advantageweighted}). 
QWR is different from AWR in the following key aspects: instead of empirical returns in the advantage estimation we train a $Q$ function (see formulas \ref{eq:awr_actor_loss} and \ref{eq:qwr_actor_loss} below for precise definition) and use sampling for the actor. QWR is different from AC-REPS as it uses deep learning for function approximation and Q-learning for fitting the critic, see Section~\ref{sec:qwr}.

Several recent works have developed algorithms similar to QWR. We provide a brief overview and ways of obtaining them from the QWR pseudocode (\autoref{algo:qwr_gradient}). AWR can be recovered by learning a value function $V(s)$ as a critic (line~\ref{ln:q-learning}) and sampling actions from the replay buffer (lines~\ref{ln:sampling-critic} and \ref{ln:sampling-actor} in Algorithm~\ref{algo:qwr_gradient}). AWAC \citep{awac} modifies AWR by learning a Q-function for the critic. We get it from QWR by sampling actions from the replay buffer (lines~\ref{ln:sampling-critic} and \ref{ln:sampling-actor}). Note that compared to AWAC, by sampling multiple actions for each state, QWR is able to take advantage of Q-learning to improve the critic. CRR \citep{crr} augments AWAC with training a distributional Q-function in line~\ref{ln:q-learning} and substituting different functions for computing advantage weights in line~\ref{ln:actor} \footnote{CRR sets the advantage weight function $f$ to be a hyperparameter in $\log \pi_\theta(a_j | s) f(Q_\phi(s, a_j) - \hat{V})$ (line~\ref{ln:actor}). In QWR, $f(x) = \exp(x / \beta)$.}. Again, compared to CRR, QWR samples multiple actions for each state, and so can take advantage of Q-learning. In a way similar to QWR, MPO \citep{abdolmaleki2018mpo} samples actions during actor training to improve generalization. Compared to QWR, it introduces a dual function for dynamically tuning $\beta$ in line~\ref{ln:actor}, adds a prior regularization for policy training and trains the critic using Retrace \citep{retrace} targets in line~\ref{ln:backup}. QWR can be thought of as a significant simplification of MPO, with addition of Q-learning to provide a better baseline for the actor. Additionally, the classical DQN \citep{dqn} algorithm for discrete action spaces can be recovered from QWR by removing the actor training loop (lines~\ref{ln:actor-begin}-\ref{ln:actor-end}), computing a maximum over all actions in Q-network training (line~\ref{ln:backup}) and using an epsilon-greedy policy w.r.t. the Q-network for data collection.




\paragraph{Offline reinforcement learning.} Offline RL is the main topic of the survey \citet{levine2020offline}. The authors state that ``offline reinforcement
learning methods equipped with powerful function approximation may enable data to be turned
into generalizable and powerful decision making engines''. We see this as one of the major challenges of modern RL and this work contributes to this challenge. Many current algorithms perform to some degree in offline RL, e.g., variants of DDPG and DQN developed by \citet{fujimoto2018offpolicy, agarwal2019striving}, 
as well as the MPO algorithm by \citet{abdolmaleki2018mpo} are promising alternatives to AWR and QWR analyzed in this work.

ABM \citep{abm} is a method of extending RL algorithms based on policy networks to offline settings. It first learns a prior policy network on the offline dataset using a loss similar to \autoref{eq:awr_actor_loss}, and then learns the final policy network using any algorithm, adding an auxiliary term penalizing KL-divergence from the prior policy. CQL \citep{cql} is a method of extending RL algorithms based on Q-networks to offline settings by introducing an auxiliary loss. To compute the loss, CQL samples actions on-line during training of the Q-network, similar to line~\ref{ln:q-learning} in QWR. EMaQ \citep{emaq} learns an ensemble of Q-functions using an Expected-Max backup operator and uses it during evaluation to pick the best action. The Q-network training part is similar to QWR with $F = \max$ in line~\ref{ln:backup} in Algorithm~\ref{algo:qwr_gradient}.


The imitation learning algorithm MARWIL by \citet{wang2018marwil} confirms that the advantage-weighted regression performs well in the context of complex games.

\begin{table*}[t]
\centering
\begin{tabular}{ l | c c c c }
Algorithm & Half-Cheetah & Walker & Hopper & Humanoid \\
\hline
QWR-LSE & $2323 \pm 332$ & $\mathbf{1301 \pm 445}$ & $\mathbf{1758 \pm 735}$ & $511 \pm 57$ \\
QWR-MAX & $2250 \pm 254$ & $1019 \pm 1185$ & $1187 \pm 345$ & $503 \pm 49$ \\
QWR-AVG & $1691 \pm 682$ & $1052 \pm 231$  & $420 \pm 65$ & $455 \pm 41$ \\
AWR & $-0.4 \pm 0$ & $67 \pm 11$ & $110 \pm 81$ & $500 \pm 4$ \\
SAC & $\mathbf{5492 \pm 8}$ & $493 \pm 6$ & $1197 \pm 175$ & $\mathbf{645 \pm 27}$ \\
PPO & $51 \pm 41$ & $-14 \pm 98$ & $15 \pm 75$ & $72 \pm 18$
\end{tabular}
\caption{Comparison of the variants of QWR with AWR \citep{peng2019advantageweighted}, SAC \citep{haarnoja2018soft} and PPO \citep{schulman2017proximal} on 4 MuJoCo environments at 100K samples. We report the median of 5 runs, $\pm$ half of the interquartile range.}
\label{tab:mujoco_ppo_sac_awr_qwr_100k}
\end{table*}

\begin{table*}[t]
\centering
\begin{tabular}{ l | c c c c c c }
Algorithm & Boxing & Breakout & Freeway & Gopher & Pong & Seaquest \\
\hline
QWR-LSE & ${\bf 4.6}$ & ${\bf 8}$ & $21.2$ & ${\bf 776}$ & $-7.6$ & $308$ \\
QWR-MAX & $-1.8$ & $0.8$ & $16.8$ & $580$ & ${\bf -2}$ & $252$ \\
QWR-AVG & $-0.8$ & $1.4$ & $19.2$ & $548$ & $-9$ &  $296$ \\
PPO & $-3.9$ & $5.9$ & $8$ & $246$ & $-20.5$ & ${\bf 370}$ \\
OTRainbow & $2.5$ & $1.9$ & ${\bf 27.9}$ & $349.5$ & $-19.3$ & $354.1$ \\
\hline
MPR & $16.1$ & $14.2$ & $23.1$ & $341.5$ & $-10.5$ & $361.8$ \\
MPR-aug & $30.5$ & $15.6$ & $24.6$ & $593.4$ & $-3.8$ & $603.8$ \\
SimPLe & $9.1$ & $16.4$ & $20.3$ & $845.6$ & $12.8$ & $683.3$ \\
\hline
Random & $0.1$ & $1.7$ & $0$ & $257.6$ & $-20.7$ & $68.4$ \\
\end{tabular}
\caption{Comparison of the variants of QWR with the sample-efficient variant of Rainbow \citep{rainbow, otrainbow}, MPR \citep{mpr}, SimPLe \citep{simple} and random scores on 6 Atari games at 100K samples. We report results of the the augmented and non-augmented version of the MPR algorithm. Since MPR and SimPLe are based on learning a model of the environment, we do not consider them when choosing the best scores.}
\label{tab:atari_100k}
\end{table*}

\begin{table*}[t]
\centering
\begin{tabular}{ l | c c c c }
Algorithm & Half-Cheetah & Walker & Hopper & Humanoid \\
\hline
QWR-LSE & $4511 \pm 85$ & $\mathbf{4558 \pm 83}$ & $\mathbf{3359 \pm 890}$ & $\mathbf{5675 \pm 236}$ \\
AWR & $2506 \pm 165$ & $1668 \pm 353$ & $1533 \pm 89$ & $639 \pm 20$ \\
SAC & $\mathbf{10433 \pm 224}$ & $4146 \pm 110$ & $3167 \pm 897$ & $5376 \pm 154$ \\
PPO & $1555 \pm 9$ & $1155 \pm 13$ & $1322 \pm 294$ & $1567 \pm 178$
\end{tabular}
\caption{Comparison of QWR-LSE with AWR, SAC and PPO on 4 MuJoCo environments at 1M samples.}
\label{tab:mujoco_1m}
\end{table*}

\begin{figure*}[t]
\centering
\begin{subfigure}{0.33\textwidth}
  \centering
  \includegraphics[width=\textwidth,trim={0 0 50cm 0},clip]{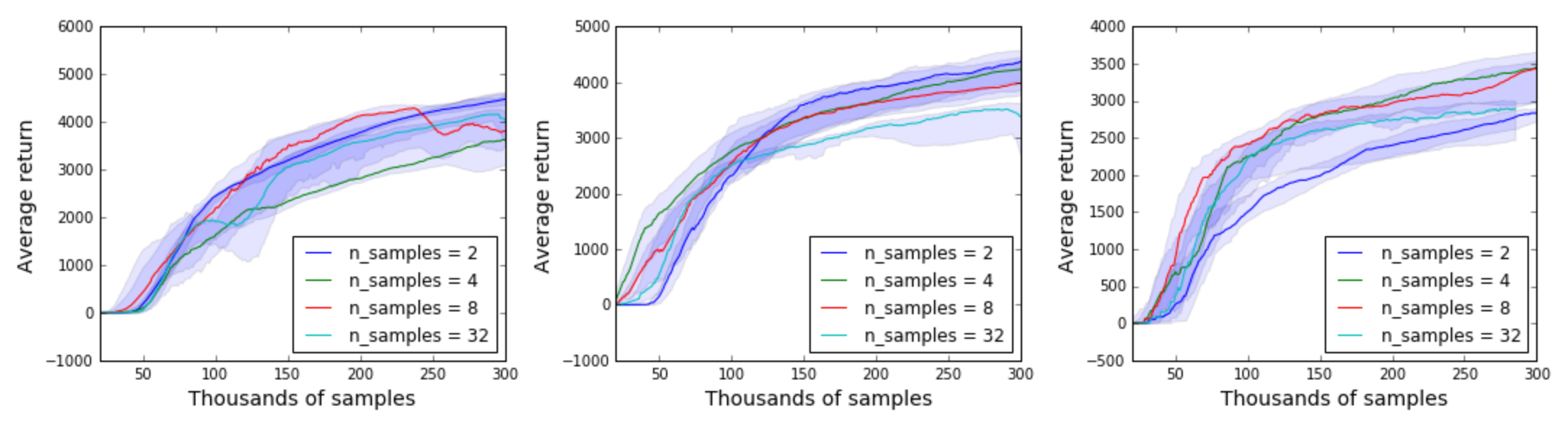}
  \caption{QWR-LSE, margin 1.}
  \label{fig:qwr_lse_1}
\end{subfigure}%
\begin{subfigure}{0.33\textwidth}
  \centering
  \includegraphics[width=\textwidth,trim={25cm 0 25cm 0},clip]{figures/hc_lse.png}
  \caption{QWR-LSE, margin 3.}
  \label{fig:qwr_lse_2}
\end{subfigure}%
\begin{subfigure}{0.33\textwidth}
  \centering
  \includegraphics[width=\textwidth,trim={50cm 0 0 0},clip]{figures/hc_lse.png}
  \caption{QWR-LSE, margin 7.}
  \label{fig:qwr_lse_3}
\end{subfigure}
\begin{subfigure}{0.33\textwidth}
  \centering
  \includegraphics[width=\textwidth,trim={0 0 50cm 0},clip]{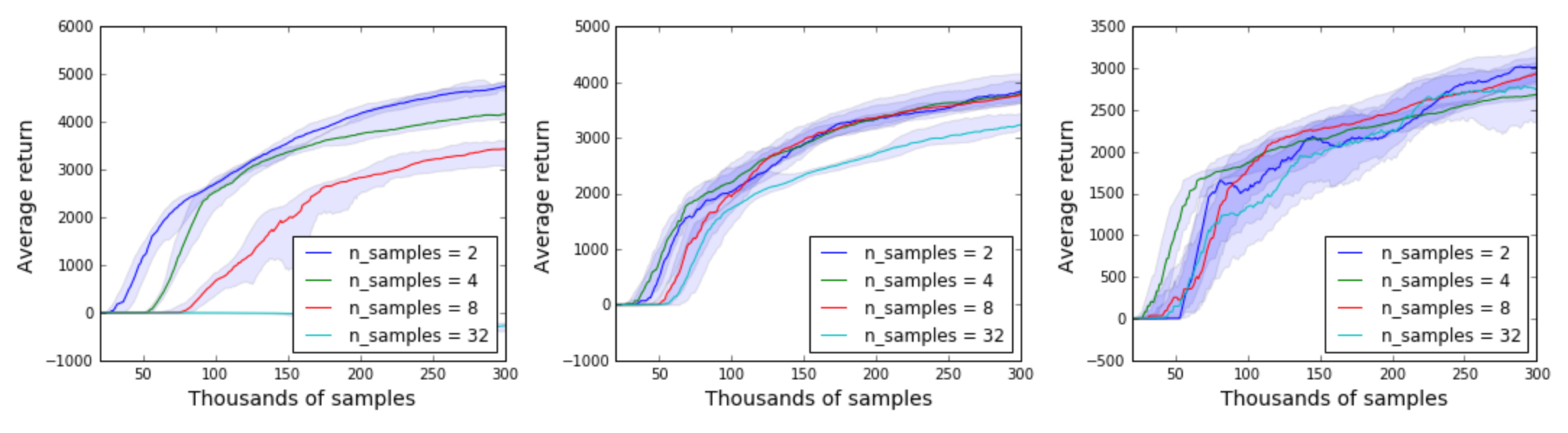}
  \caption{QWR-MAX, margin 1.}
  \label{fig:qwr_max_1}
\end{subfigure}%
\begin{subfigure}{0.33\textwidth}
  \centering
  \includegraphics[width=\textwidth,trim={25cm 0 25cm 0},clip]{figures/hc_max.png}
  \caption{QWR-MAX, margin 3.}
  \label{fig:qwr_max_2}
\end{subfigure}%
\begin{subfigure}{0.33\textwidth}
  \centering
  \includegraphics[width=\textwidth,trim={50cm 0 0 0},clip]{figures/hc_max.png}
  \caption{QWR-MAX, margin 7.}
  \label{fig:qwr_max_3}
\end{subfigure}%
\caption{Ablation of QWR with respect to the margin, the number of action samples and the method of training the critic. The results are shown on the Half-Cheetah environment. The plots show the median of 5 runs with the shaded area denoting the interquartile range.}
\label{fig:qwr_ablation}
\end{figure*}

\section{Experiments}

\paragraph{Neural architectures.} In all MuJoCo experiments, for both value and policy networks, we use multi-layer perceptrons with two layers 256 neurons each, and ReLU activations. In all Atari experiments, for both value and policy networks, we use the same convolutional architectures as in \citet{dqn}. To feed actions to the network, we embed them using one linear layer, connected to the rest of the network using the formula $o \cdot \tanh(a)$ where $o$ is the processed observation and $a$ is the embedded action. This is followed by the value or policy head. For the policy, we parameterize either the log-probabilities of actions in case of discrete action spaces, or the mean of a Gaussian distribution in case of continuous action spaces, while keeping the standard deviation constant, as $0.4$.

\subsection{Sample efficiency}
Since we are concerned with sample efficiency, we focus our first experiments on
the case when the number of interactions with the environment is limited. To use a single
number that allows comparisons with previous work both on MuJoCo and Atari, we decided to
restrict the number of interactions to 100K. This number is high enough, that the state-of-the-art algorithms such as SAC reach good performance.

We run experiments on 4  MuJoCo environments and 6 Atari games, evaluating three versions of QWR with the 3 backup operators introduced in Section~\ref{qwr:intro}: QWR-LSE (using log-sum-exp), QWR-MAX (using maximum) and QWR-AVG (using average). For all experiments, we set the Q target truncation horizon $T$ to 3. In MuJoCo experiments, we set the number of action samples $k$ to 4. In Atari experiments, because of the discrete action space, we can compute the policy loss for each transition explicitly, without sampling. All other hyperparameters are kept the same between those domains. We discuss the choice of $T, k$ and show ablations in \autoref{sec:ablations}, while more experimental details are given in Appendix~\ref{sec:expdetail}.

In Tables~\ref{tab:mujoco_ppo_sac_awr_qwr_100k} and \ref{tab:atari_100k} we present the final returns at 100K samples for the considered
algorithms and environments. To put them within a context, we also provide those results for SAC, PPO, OTRainbow - a variant of Rainbow tuned for sample efficiency, MPR and SimPLe.

On all considered MuJoCo tasks, QWR exceeds the performance of AWR and PPO. The better sample efficiency is particularly well visible in the case of Walker, where each variant of QWR performs better than any baseline considered. On Hopper, QWR-LSE - the best variant - outpaces all baselines by a large margin. On Humanoid, it comes close to SAC - the state of the art on MuJoCo.

QWR surpasses PPO and Rainbow in 4 out of 6 Atari games. In Gopher and Pong QWR outperforms even against the augmented and non-augmented versions of the model-based MPR algorithm.

\subsection{More samples}

To verify that our algorithm makes a good use of higher sample budgets, we also evaluate it on the 4 MuJoCo tasks at 1M samples. For the purpose of this experiment, we adapt several of the hyperparameters of QWR to the larger amount of data. The details are provided in Appendix~\ref{sec:expdetail}. We present the results in \autoref{tab:mujoco_1m}.

On Walker, Hopper and Humanoid, QWR outperforms all baselines. Only on Half-Cheetah it is surpassed by SAC. In all tasks, QWR achieves significantly higher scores than AWR and PPO, which shows that the sample-efficiency improvements applied in QWR translate well to the higher budget of 1M samples.

\subsection{Ablations}
\label{sec:ablations}

In \autoref{fig:qwr_ablation} we provide an ablation of QWR with respect to the backup method $F$, multistep target horizon $T$ ("margin") and the number of action samples $k$ to consider when training the actor and the critic. As we can see, the algorithm is fairly robust to the choice of these hyperparameters.

In total, the log-sum-exp backup (LSE) achieves the best results -- compare \autoref{fig:qwr_lse_2} and \autoref{fig:qwr_max_2}. Max backup performs well with margin 1, but is more sensitive to higher numbers of samples -- compare \autoref{fig:qwr_max_1} and \autoref{fig:qwr_max_2}.
The log-sum-exp backup is less vulnerable to this effect -- compare \autoref{fig:qwr_lse_1} and \autoref{fig:qwr_max_1}. Higher margins decrease performance -- see \autoref{fig:qwr_lse_3} and \autoref{fig:qwr_lse_2}. We conjecture this to be due to stale action sequences in the replay buffer biasing the multi-step targets. Again, the log-sum-exp backup is less prone to this issue -- compare \autoref{fig:qwr_lse_3} to \autoref{fig:qwr_max_3}.




\subsection{Offline RL}

Both QWR and AWR are capable of handling expert data.
AWR was shown to behave in a stable way when provided only with a number of expert trajectories (see Figure 7 in \citet{peng2019advantageweighted}) without additional data collection. In this respect, the performance of AWR is much more robust than the performance of PPO and SAC. In Figure~\ref{fig:awr_qwr_bc} we show the same result for QWR -- in terms of re-using the expert trajectories, it matches or exceeds AWR. The QWR trainings based on offline data were remarkably stable and worked well across all environments we have tried.

\begin{figure*}
  \centering
  \begin{subfigure}{.3\textwidth}
  \centering
  \includegraphics[width=\linewidth,height=7em]{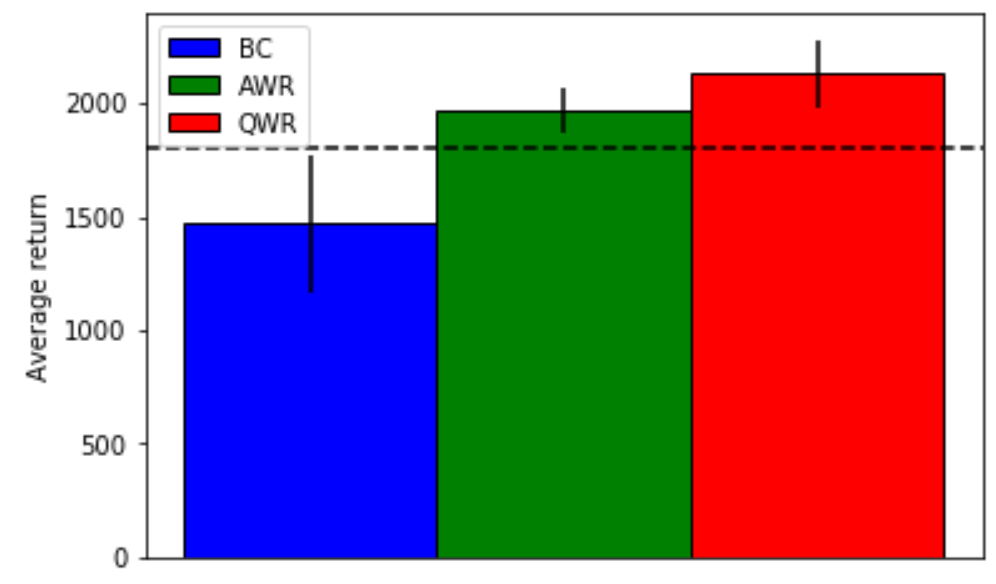}
  \caption{HalfCheetah}
  \label{fig:sub1bc}
\end{subfigure}%
\begin{subfigure}{.3\textwidth}
  \centering
  \includegraphics[width=\linewidth,height=7em]{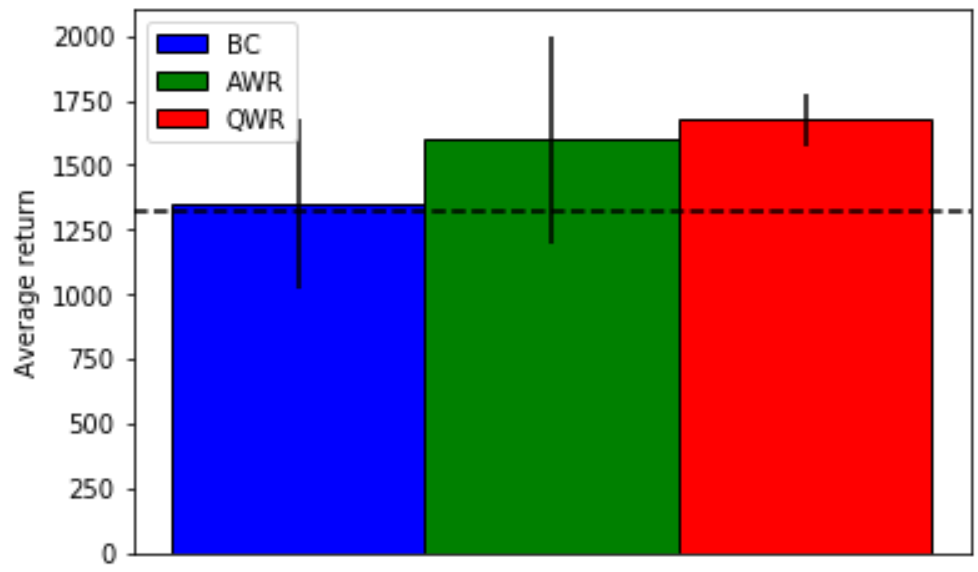}
  \caption{Hopper}
  \label{fig:sub2bc}
\end{subfigure}%
  \begin{subfigure}{.3\textwidth}
  \centering
  \includegraphics[width=\linewidth,height=7em]{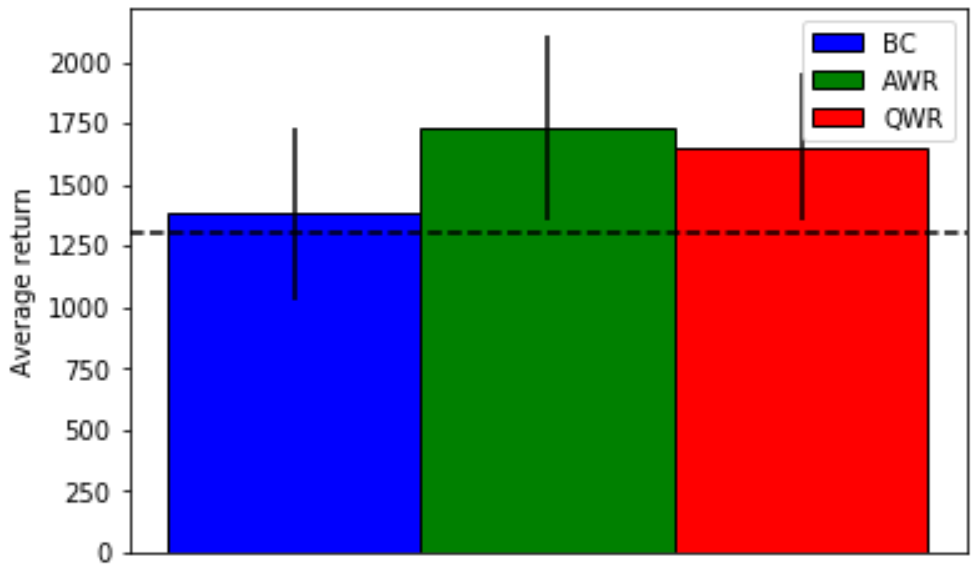}
  \caption{Walker2d}
  \label{fig:sub3bc}
\end{subfigure}

\begin{subfigure}{.3\textwidth}
  \centering
  \includegraphics[width=\linewidth,height=7em]{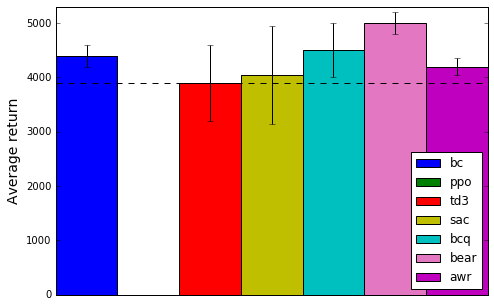}
  \caption{HalfCheetah}
  \label{fig:sub4bc}
\end{subfigure}%
\begin{subfigure}{.3\textwidth}
  \centering
  \includegraphics[width=\linewidth,height=7em]{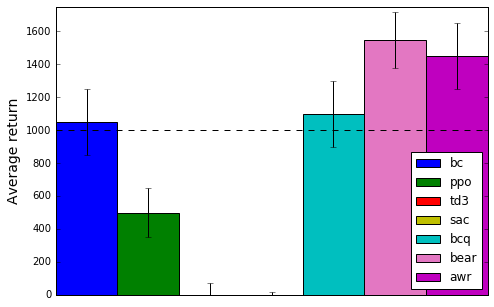}
  \caption{Hopper}
  \label{fig:sub5bc}
\end{subfigure}%
  \begin{subfigure}{.3\textwidth}
  \centering
  \includegraphics[width=\linewidth,height=7em]{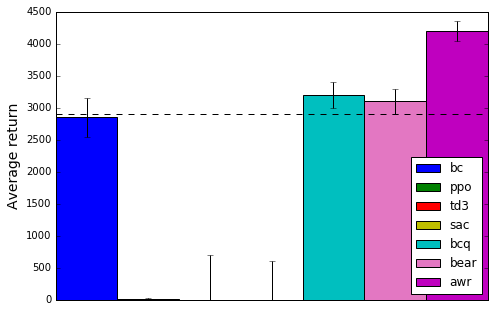}
  \caption{Walker2d}
  \label{fig:sub6bc}
\end{subfigure}
  \caption{Figures \ref{fig:sub1bc}, \ref{fig:sub2bc} and \ref{fig:sub3bc} show offline trainings based on 50 trajectories of length 1000 collected by diverse policies. The horizontal lines mark the average return of a policy from the dataset. The bars denote median returns out of 4 runs, and the vertical lines denote the interquartile range. Data for figures \ref{fig:sub4bc}, \ref{fig:sub5bc} and \ref{fig:sub6bc} is borrowed from \citet{peng2019advantageweighted} to cover a broader family of algorithms and show that offline training fails for many RL algorithms.}
\label{fig:awr_qwr_bc}
\end{figure*}

For the offline RL experiments, we have trained each algorithm for 30 iterations, without additional data collection. The training trajectories contained only states, actions and rewards, without any algorithm-specific data. In QWR, we have set the per-step sampling policies $\mu$ to be Gaussians with mean at the performed action and standard deviation set to $0.4$, same as in \citet{peng2019advantageweighted}.

\section{Discussion and Future Work}

We present Q-value Weighted Regression (QWR), an off-policy actor-critic algorithm that extends
Advantage Weighted Regression with action sampling and Q-learning. It is significantly more sample-efficient than AWR and works well with discrete actions and in visual domains, e.g., on Atari games. QWR consists of two interleaved steps of supervised training: the critic learning the Q function using a predefined backup operator, and the actor learning the policy with weighted regression based on multiple sampled actions. Thanks to this clear structure, QWR is simple to implement and debug. It is also stable in a wide range of hyperparameter choices and works well in the offline setting.

Importantly, we designed QWR thanks to a theoretical analysis that revealed why AWR may not work when there are limits on data collection in the environment.
Our analysis for the limited data regime is based on the \emph{state-determines-action} assumption
that allows to fully solve AWR analytically while still being realistic and indicative
of the performance of this algorithm with few samples.
We believe that using the \emph{state-determines-action} assumption can yield important insights into other RL algorithms as well.

QWR already achieves state-of-the-art results in settings with limited data and we believe that it can be further improved in the future. The critic training could benefit from the advances in Q-learning methods such as double Q-networks \citep{hasselt2015deep} or Polyak averaging \citep{Polyak1990NewMO}, already used in SAC. Distributional Q-learning \citet{bellemare2017distributional} and the use of ensembles like REM \citet{agarwal2020optimistic} could yield further improvements. 

Notably, the QWR results at 100K that we present are achieved with the same set of hyperparameters (except for the network architecture) both for MuJoCo environments and for Atari games. This is rare among deep reinforcement learning algorithms, especially among ones that strive for sample-efficiency. Combined with its stability and good performance in offline settings, this makes
QWR a compelling choice for reinforcement learning in domains with limited data.

\newpage
\bibliography{references.bib}
\bibliographystyle{icml2021}

\section{Appendix}

\subsection{Experimental Details}
\label{sec:expdetail}

We run experiments on 4 \href{http://www.mujoco.org/}{MuJoCo} environments: \emph{Half-Cheetah}, \emph{Walker}, \emph{Hopper} and \emph{Humanoid} and on 6 \href{https://github.com/mgbellemare/Arcade-Learning-Environment}{Atari} games: \emph{Boxing}, \emph{Breakout}, \emph{Freeway}, \emph{Gopher}, \emph{Pong} and \emph{Seaquest}. For the MuJoCo environments, we limit the episode length to $1000$. For the Atari environments, we apply the following preprocessing:

\begin{itemize}
    \item Repeating each action for 4 consecutive steps, taking a maximum of 2 last frames as the observation.
    \item Stacking 4 last frames obtained from the previous step in one observation.
    \item Gray-scale observations, cropped and rescaled to size $84 \times 84$.
    \item Maximum $10\mathrm{K}$ interactions per episode.
    \item Random number of no-op actions from range $[0..30]$ at the beginning of each episode.
    \item Rewards clipped to the $[-1, 1]$ range during training.
\end{itemize}

Our code with the exact configurations we use to reproduce the experiments is available as open source\footnote{\url{url_removed_to_preserve_anonymity}}.
We use the same hyperparameters for the MuJoCo and Atari experiments, and almost the same hyperparameters for the 100K and 1M sample budgets. The hyperparameters and their tuning ranges are reported in \autoref{tab:hparams}.

Before calculating the actor loss, we normalize the advantages over the entire batch by subtracting their mean and dividing by their standard deviation, same as \citet{peng2019advantageweighted}. We perform a similar procedure for the log-sum-exp backup operator used in critic training. Before applying the backup, we divide the Q-values by a computed measure of their scale $s$. After applying the backup, we re-scale the target by $s$. There is no need to subtract the mean, as log-sum-exp is translation-invariant.
\begin{equation}
    F(X) = \tau s(X) \log \left[ \frac{1}{|X|} \sum_{x \in X} \exp\left(\frac{x}{\tau s(X)}\right) \right]
\end{equation}
The parameters of this backup are the only ones different between the 100K and 1M experiments. For 100K, we use $\tau = 0.3$ and mean absolute deviation as $s(X)$. For 1M, we use $\tau = 1.0$ and standard deviation as $s(X)$.

\begin{table*}[t]
\centering
\begin{tabularx}{\textwidth}{ l | X X }
Hyperparameter & Value & Considered range \\
\hline
$k$ - number of action samples & $4$ & $\{2, 4, 8, 32\}$ \\
$T$ - multi-step target horizon ("margin") & $3$ & $\{1, 3, 7\}$ \\
$\beta$ - actor loss temperature & $1$ & $\{0.1, 0.3, 1, 3, 10\}$ \\
$F$ - critic backup operator & log-sum-exp & mean, log-sum-exp, $\max$ \\
$\gamma$ - discount factor for the returns & $0.99$ & $\{0.99\}$ \\
$\lambda$ - discount factor in TD($\lambda$) & $0.95$ & $\{0.95\}$ \\
$\alpha_\pi$ - actor learning rate & \num{1e-4} & $\{\num{1e-4}, \num{2e-4}, \num{5e-4}\}$ \\
$\alpha_V$ - critic learning rate & \num{5e-4} & $\{\num{2e-4}, \num{5e-4}, \num{1e-3}\}$ \\
batch size (actor and critic) & $256$ & $\{128, 256, 512\}$ \\
replay buffer size & $50\mathrm{K}$ interactions & $\{20\mathrm{K}, 50\mathrm{K}, 100\mathrm{K}, 200\mathrm{K}\}$ \\
\texttt{n\_actor\_steps} & $1000$ & $\{1000, 2000, 3000\}$ \\
\texttt{n\_critic\_steps} & $1000$ & $\{200, 500, 1000, 2000\}$ \\
\texttt{update\_frequency} & $100$ & $\{100, 200, 300\}$ \\
\texttt{n\_iterations} & \multicolumn{2}{p{10cm}}{Until we reach the desired number of interactions. In all experiments, we collect $1000$ interactions with the environment in each iteration of the algorithm.} \\
\end{tabularx}
\caption{Hyperparameter values and considered ranges.}
\label{tab:hparams}
\end{table*}

When training the networks, we use the Adam optimizer.  We use the standard architectures for deep networks. In MuJoCo experiments we use a multi-layer perceptron with two layers 256 neurons each and ReLU activations. In Atari experiments we use the same convolutional architectures as \citet{dqn}.

The 100K experiments took around 18 hours each, on a single \href{https://cloud.google.com/tpu/docs/system-architecture}{TPU v2 chip}. The 1M experiments took around 180 hours each, using the same hardware.

\subsection{Formal Analysis of AWR with Limited Data}
\label{sec:limdata}

Since sample efficiency is one of the key challenges in deep reinforcement learning,
it would be desirable to have better tools to understand why any RL algorithm -- for instance AWR -- is
sample efficient or not. This is hard to achieve in the general setting, but we identify
a key simplifying assumption that allows us to solve AWR analytically and identify the source
of its problems.

The assumption we introduce, called \emph{state-determines-action}, concerns the content of the
replay buffer $\mathcal{D}$ of an off-policy RL algorithm. The replay buffer contains all state-action pairs
that the algorithm has visited so far during its interactions with the environment. We say that
a replay buffer $\mathcal{D}$ satisfies the \emph{state-determines-action} assumption when for each
state $s$ in the buffer, there is a unique action that was taken from it, formally:
\[ \text{for all } (s, a), (s', a') \in \mathcal{D}: s = s' \implies a = a'. \]

A simplifying assumption like \emph{state-determines-action} is useful only if it indeed simplifies the analysis of RL algorithms.
We show that in case of AWR it does even more -- it allows us to
analytically calculate the final policy that the algorithm produces.
In the case of AWR, it turns out that the resulting policy yields no improvement over the sampling policy.


While AWR achieves very good results after longer training, it is not very sample efficient,
as noted in the future work section of \citep{peng2019advantageweighted}. To address this
problem, let us analyze a single loop of actor training in AWR:
\begin{align} \label{eqn:awrupdate}
\pi^{i + 1}_\mathcal{D} \gets \argmax_\pi \mathbb{E}_{\mathbf{s, a} \sim \mathcal{D}} \log \pi(\mathbf{a} | \mathbf{s})~\xi^{\mathbf{s, a}}_\mathcal{D}, & \\
\textrm{where}~\xi^{\mathbf{s, a}}_\mathcal{D} = \exp \left( \frac{1}{\beta}(\mathcal{R}_\mathcal{D}^{\mathbf{s, a}} - V^i_\mathcal{D}(\mathbf{s})) \right). &
\end{align}

How does this update act on a replay buffer that satisfies the \emph{state-determines-action} assumption?
It turns out that we can answer this question analytically using the following theorem.

\begin{theorem} \label{thm:awr_disc}
Let $\mathcal{A}$ be a discrete action space. Let a replay buffer $\mathcal{D} \subseteq \mathcal{S} \times \mathcal{A}$ satisfy the state-determines-action assumption. Let $\pi_\mathcal{D}$ be the probability function of a distribution that clones the behavior from $\mathcal{D}$, i.e., that assigns to each state $s$ from $\mathcal{D}$ the action $a$ such that
$(s, a) \in \mathcal{D}$ with probability $1$.  Then, under the AWR update,
$\pi^{i + 1}_\mathcal{D} \gets \pi_\mathcal{D}$.
\end{theorem}

\begin{proof}
By definition of the AWR update rule, $\pi^{i + 1}_\mathcal{D} \gets \argmax_\pi \mathbb{E}_{\mathbf{s, a} \sim \mathcal{D}} [ \log \pi(\mathbf{a} | \mathbf{s}) \xi^{\mathbf{s, a}}_\mathcal{D} ]$. Recall that the number $\xi^{\mathbf{s, a}}_\mathcal{D} > 0$ as an exponent of another number is always positive.
Since $\mathcal{A}$ is a discrete action space, $\pi$ is a 
discrete policy
and
we have $\pi(\mathbf{a} | \mathbf{s}) \leq 1$,
so $\log \pi(\mathbf{a} | \mathbf{s})$ in the considered equation is at most 0
($\log$ is a strictly increasing function and $\log(1)=0$).
Thus the value $\log \pi(\mathbf{a} | \mathbf{s})~\xi^{\mathbf{s, a}}_\mathcal{D}$ can be at most $0$
and it reaches its maximum value for the policy that assigns probability $1$ to the action $a$ in state $s$
for each $(s, a) \in \mathcal{D}$. Therefore $\pi_\mathcal{D}$ attains the $\argmax_\pi$ as required.
\end{proof}




As we can see from the above theorem, the AWR update rule will insist on cloning the action
taken in the replay buffer as long as it satisfies the \emph{state-determines-action} assumption.
In the extreme case of a deterministic environment, the new policy $\pi_\mathcal{D}$ will not add any new data
to the buffer, only replay a trajectory already in it. So the whole AWR loop will end
with the policy $\pi_\mathcal{D}$, which yields no improvement.

In the next section, we prove an analogous theorem for continuous action spaces.

\subsubsection{Continuous action spaces}
The statement of Theorem~\ref{thm:awr_disc} must be adjusted for the case of continuous actions.
First of all, let us clarify the notation of AWR update introduced in Equation~\ref{eqn:awrupdate}:
\begin{equation*}
\pi^{i + 1}_\mathcal{D} \gets \argmax_\pi \mathbb{E}_{\mathbf{s, a} \sim \mathcal{D}} \left[ \log \pi(\mathbf{a} | \mathbf{s})~\xi^{\mathbf{s, a}}_\mathcal{D} \right]
\end{equation*}
For discrete actions, the symbol $\pi(\mathbf{a|s})$ denotes the probability function of a discrete distribution. In the continuous setting, we use it to denote probability density functions.

Now let us define the policy that "clones the behavior from the replay buffer". Intuitively that would be a distribution that concentrates most of its probability mass arbitrarily close to the action in the replay buffer.

\begin{theorem} \label{thm:awr_cont}
Let $\mathcal{A}$ be a continuous action space. Let a replay buffer $\mathcal{D} \subseteq \mathcal{S} \times \mathcal{A}$ satisfy the state-determines-action assumption.
For a given $\varepsilon > 0$ let us consider the following family of parameterized Gaussian distributions
\begin{equation*}
    \pi_{\mu, \sigma}(\mathbf{a|s}) = \frac{1}{\sigma(\mathbf{s})\sqrt{2\pi}}
    e^{-\frac{1}{2}(\frac{\mathbf{a}-\mu(\mathbf{s})}{\sigma(\mathbf{s})})^2},
\end{equation*}
where $\sigma(\mathbf{s}) \geq \varepsilon$, and define $\pi_\mathcal{D}^\epsilon = \pi_{\mu, \sigma}$ such that $\sigma(\mathbf{s}) = \varepsilon$ and $\mu(\mathbf{s}) = \mathbf{a}~$ for $\mathbf{(s, a)} \in \mathcal{D}$.
If we perform the optimization in the AWR update over such a family of distributions,
we get $\pi^{i + 1}_\mathcal{D} \gets \pi_\mathcal{D}^\epsilon$.

\end{theorem}

\begin{proof}
The reasoning is similar to the proof of Theorem~\ref{thm:awr_disc} but we cannot
rely on $\log \pi(\mathbf{a} | \mathbf{s}) \leq 0$, as probability density functions can take arbitrarily large values.
Let $s$ be any state that for some $a$ we have $(s, a) \in \mathcal{D}$.
For the assumed family of distributions we have
\begin{align*}
\log \pi(a|s) & =  \log(\frac{1}{\sigma\sqrt{2\pi}}) + \log(e^{-\frac{1}{2}(\frac{a-\mu}{\sigma})^2}) = \\
& = -\log(\sigma\sqrt{2\pi}) - \frac{1}{2}(\frac{a-\mu}{\sigma})^2.
\end{align*}
$\log \pi(a|s)$ is a quadratic function of $\mu$, so it attains the maximum value at $\mu = a$~for every $\sigma > 0$.
Now let's look at
\begin{equation*}
\frac{\partial}{\partial \sigma} \log(\frac{1}{\sigma\sqrt{2\pi}}) =
-\frac{1}{\sigma} < 0.
\end{equation*}
The derivative is negative regardless of $\sigma$, so $\log \pi(\mathbf{a} | \mathbf{s})$ is maximized for the lowest allowed $\sigma(s) = \varepsilon$ and $\mu(s) = a$.
This is true for arbitrary state-action pair such that $(s, a) \in \mathcal{D}$.
So under the AWR update we get $\pi^{i + 1}_\mathcal{D} \gets \pi_\mathcal{D}^\epsilon$.
\end{proof}

This gives us intuition that the probability distributions $\pi(\mathbf{a | s})$ commonly used in RL (e.g. Gaussian) can be improved by increasing the density at $\mathbf{a}$ and decreasing it everywhere else. For those distributions, the maximum can come arbitrarily close to the Dirac delta, where $\pi(\mathbf{a} | \mathbf{s}) = \infty$.

Given that AWR aims to copy the replay buffer, as demonstrated by Theorems~\ref{thm:awr_disc} and \ref{thm:awr_cont}, how come this algorithm works so well in practice, given enough interactions? First of all,
note that for this effect to occur, the neural network used for AWR actor must be
large enough and trained long enough to memorize the data from the replay buffer.
Furthermore, the policy it learns must be allowed to express distributions
that assign probability $1$ to a single action. This holds for environments with discrete
actions, and for continuous actions with distributions with controlled scale, but it is not
true e.g. when using Gaussian distributions with fixed variance. However, in the latter case,
the proof of \autoref{thm:awr_cont} shows that the AWR update will place the mean of the
policy distribution at the performed action, regardless of the variance, which still leads to no
improvement over the sampling policy.

In the next section, we show that using an algorithm, that corrects this cloning behavior,
leads to improved sample efficiency.

\subsection{Formal Analysis of QWR with Limited Data} \label{sec:limdata-qwr}
To see how QWR performs under limited data, we are going to formulate a positive theorem showing that it achieves the policy improvement that AWR aims for even in a limited data setting.
Note that this time we allow replay buffers that do not necessarily satisfy the \textit{state-determines-action} assumption. But, for clarity, we make a simplifying assumption that the replay buffer has been sampled by a single policy $\mu$.

Recall the QWR update rule:
\begin{align}
\begin{split}
\pi^{i + 1}_\mathcal{D} \gets \argmax_\pi \mathbb{E}_{\mathbf{s} \sim \mathcal{D}} \mathbb{E}_{\mathbf{a} \sim \mu(\mathbf{\cdot | s})} \log \pi(\mathbf{a} | \mathbf{s})~\xi_\mu^\mathbf{s, a}, \\
\begin{aligned}
\textrm{where}~\xi_\mu^\mathbf{s, a} &= \exp \left( \frac{1}{\beta} (Q_\mu(\mathbf{s, a}) - \hat{V}_\mu(\mathbf{s})) \right), \\
\hat{V}_\mu(\mathbf{s})~&=~\mathbb{E}_{a \sim \mu(\cdot | \mathbf{s})} Q_\mu(\mathbf{s, a}),
\end{aligned}
\end{split}
\end{align}
and $\mathcal{D}$ is the set of states in the replay buffer.

Let $\pi^\star_\mu(\mathbf{a | s}) \propto \mu(\mathbf{a | s}) \exp \left( \frac{1}{\beta} \left( \mathcal{R}_\mathcal{\mu}^{\mathbf{s, a}} -  V_\mu(\mathbf{s}) \right) \right)$, where $V_\mu$ is the state value function of $\mu$. This is the policy optimizing the expected improvement over the sampling policy $\mu$, subject to a KL constraint -- the same as in Equation~36 in
\citet{peng2019advantageweighted}.

Since $\pi_\mu^\star$ is the target policy resulting from the AWR derivation, we know
from~\citet{peng2019advantageweighted} that AWR will update towards this policy in the limit,
when the replay buffer is large enough. But from Theorem~\ref{thm:awr_disc} we know that it will
fail to perform this update when the state-determines-action assumption holds.
Below we show that QWR will perform the same desirable update for any replay buffer,
as long as we restrict the attention to states in the buffer.

\begin{theorem} \label{thm:qwr}
Let $\mathcal{D} \subseteq \mathcal{S}$ be a finite sample from $d_\mu(\mathbf{s})$ - the undiscounted state distribution of a policy $\mu(\mathbf{a | s})$. Let $Q_\mu$ be the state-action value function for $\mu$, so $Q_\mu(\mathbf{s, a}) = \mathcal{R}_\mathcal{\mu}^{\mathbf{s, a}}$ for any state and action. Then, under the QWR actor update, $\pi^{i + 1}_\mathcal{D} \gets \pi^\star_\mu|_\mathcal{D}$,
where $\pi^\star_\mu|_\mathcal{D}$ is the policy $\pi^\star_\mu$. restricted to the set of states $\mathcal{D}$.
\end{theorem}


\begin{proof}
Let $s$ be an arbitrary state in $\mathcal{D}$.
From the definition of $\hat{V}_\mu(\mathbf{s})$ we have
\begin{equation*}
    \hat{V}_\mu(\mathbf{s})~=~\mathbb{E}_{a \sim \mu(\cdot | \mathbf{s})} Q_\mu(\mathbf{s, a}) = V_\mu(\mathbf{s}).
\end{equation*}
Since $Q_\mu(\mathbf{s, a}) = \mathcal{R}_\mathcal{\mu}^{\mathbf{s, a}}$,
\begin{align} \label{eq:sampling_from_mu}
\begin{split}
\pi^{i + 1}_\mathcal{D} \gets~&\argmax_\pi \mathbb{E}_{\mathbf{s} \sim \mathcal{D}} \mathbb{E}_{\mathbf{a} \sim \mu(\mathbf{\cdot | s})} \log \pi(\mathbf{a} | \mathbf{s})~\xi_\mu^\mathbf{s, a} \\
\textrm{where}~\xi_\mu^\mathbf{s, a} &= \exp \left( \frac{1}{\beta} (\mathcal{R}_\mathcal{\mu}^{\mathbf{s, a}} - V_\mu(\mathbf{s})) \right).
\end{split}
\end{align}

We can now change the measure using the definition of $\pi^\star$:
\begin{equation}
\pi^{i + 1}_\mathcal{D} \gets \argmax_\pi \mathbb{E}_{\mathbf{s} \sim \mathcal{D}} \mathbb{E}_{\mathbf{a} \sim \pi^\star(\mathbf{\cdot | s})}  \log \pi(\mathbf{a} | \mathbf{s}) \mathrm{.}
\end{equation}

The inner expectation, up to a normalizing constant, is the negative cross-entropy between $\pi^\star$ and $\pi$. Since cross-entropy between two distributions is minimized when the distributions are equal, the optimum is reached at $\pi = \pi^\star_\mu|_\mathcal{D}$ for all $s \in \mathcal{D}$.
\end{proof}


As we can see, the QWR update rule reaches the desired target policy even under limited data. This stands in contrast to AWR, which requires repeating states in the replay buffer, as shown in Theorems~\ref{thm:awr_disc} and \ref{thm:awr_cont}.

\subsection{Multi-step targets}

\label{sec:multistep}

To make the training of the Q-value network more efficient, we implement an approach inspired by widely-used multi-step Q-learning \citep{mnih2016asynchronous}. We consider targets for the Q-value network computed over multiple different time horizons:

\begin{align}
\begin{split}
Q^\star_{\mu, t}(\mathbf{s}_i, \mathbf{a}_i) =& \left( \sum_{j = i}^{i + t - 1} \gamma^j \mathbf{r}_j \right) + \gamma^{i + t} \mathbb{E}_{\mathbf{a}'_1, ..., \mathbf{a}'_n \sim \mu(\cdot | \mathbf{s}_{i + t})} \\
&F(\{ Q_\mu(\mathbf{s}_{t + 1}, \mathbf{a}'_1), ..., Q_\mu(\mathbf{s}_{t + 1}, \mathbf{a}'_n) \})
\end{split}
\end{align}
where $\mathbf{s}_i$, $\mathbf{a}_i$, $\mathbf{r}_i$ are the states, actions and rewards in a collected trajectory, respectively. We aggregate those multi-step targets using a truncated TD($\lambda$) estimator \citep[p.~236]{sutton_barto}:
\begin{equation}
Q^\star_\mu(\mathbf{s, a}) = (1 - \lambda) \sum_{t = 1}^T \lambda^{t - 1} Q^\star_{\mu, t}(\mathbf{s, a})
\end{equation}

\end{document}